\documentclass{article}

\usepackage{amsmath, amsthm} 
\usepackage{amsfonts}
\usepackage{amssymb}

\usepackage[english]{babel}
\usepackage{natbib}
\usepackage{bookmark}
\usepackage{booktabs}

\usepackage{enumerate}

\usepackage{graphicx}

\usepackage{hyperref}

\usepackage{multirow}

\newcommand{\reals}{\ensuremath{\mathbb{R}}}

\newcommand{\E}{\ensuremath{\mathbb{E}}}
\renewcommand{\P}{\ensuremath{\mathbb{P}}}
\newcommand{\var}[1]{\textup{Var}\left(#1\right)}

\renewcommand{\epsilon}{\varepsilon}

\newcommand{\ind}{{\mathbb{I}}}

\newcommand{\regret}{\textup{regret}}

\newcommand{\X}{\ensuremath{\mathbf{X}}}

\newtheorem{theorem}{Theorem}

\newtheorem{prop}[theorem]{Proposition}

\newtheorem{lemma}[theorem]{Lemma}
\newtheorem{definition}{Definition}
\newtheorem{assumption}{Assumption}

\renewcommand{\[}{\begin{equation}}
\renewcommand{\]}{\end{equation}}

\title{Learning with Abandonment}
\author{Ramesh Johari and Sven Schmit\\Stanford University}
\date{\today}

\begin{document}
    \maketitle

    \begin{abstract}
  Consider a platform that wants to learn a personalized policy
  for each user, but the platform faces the risk of a user abandoning the platform
  if she is dissatisfied with the actions of the platform.
  For example, a platform is interested in personalizing the number of
  newsletters it sends, but faces the risk that the user unsubscribes forever.
  We propose a general thresholded learning model for scenarios like this,
  and discuss the structure of optimal policies.
  We describe salient features of optimal personalization algorithms
  and how feedback the platform receives impacts the results.
  Furthermore, we investigate how the platform can efficiently learn the heterogeneity
  across users by interacting with a population
  and provide performance guarantees.
\end{abstract}

    \section{Introduction}

Machine learning algorithms are increasingly intermediating interactions
between platforms and their users.  As a result, users' interaction with the
algorithms will impact optimal learning strategies; we investigate this
consequence in our work.  In the setting we consider, a platform wants to
personalize service to each user.  The distinctive feature in this work is that
the platform faces the risk of a user abandoning the platform if she is
dissatisfied with the actions of the platform.
Algorithms designed by the platform thus need to be careful
to avoid losing users.

There are many examples of such settings.
In the near future, smart energy meters will be able to
throttle consumers' energy consumption to increase efficiency of the power grid
during peak demand, e.g., by raising or lowering the level of air conditioning.
This can lead to cost savings for both utility companies and consumers.
However, if the utility company is too aggressive in its throttling of
energy, a user might abandon the program.
Due to heterogeneity in housing, appliances and preferences of customers,
it is important that utility companies learn personalized strategies
for each consumer.

Content creators (e.g., news sites, blogs, etc.) face a similar problem with e-mail dissemination.
There is value in sending more e-mails, but each e-mail also risks
the recipient unsubscribing, taking away any opportunity of the creator
to interact with the user in the future.
Yet another example is that of mobile app notifications.
These can be used to improve user engagement and experience.
However if the platform sends too many notifications, an upset user might turn
off notifications from the application.

In all of the above scenarios, we face a decision problem where ``more is better;''
however, there is a threshold beyond which the user abandons
and no further rewards are gained.
This work focuses on developing insight into the structure of optimal learning strategies in such settings.  We are particularly interested in understanding when such strategies take on a ``simple'' structure, as we elaborate below.

In Section~\ref{sec:model}, we introduce a benchmark model of learning with
abandonment.  In the initial model we consider, a platform interacts with a
single user over time.  The user has a {\em threshold} $\theta$ drawn from a
distribution $F$, and at each time $t = 0, 1, 2, \ldots$ the platform chooses
an action $x_t$.  If $x_t$ ever exceeds $\theta$, the user abandons; otherwise,
the user stays, and the platform earns some reward dependent on $x_t$.

We first consider the case where the distribution $F$ and the reward function are known
known (say, from prior estimation), and the challenge is finding an optimal
strategy for a given new user.  We consider the problem of maximizing expected
discounted reward.  Intuitively, we might expect that the optimal policy is
increasing and depends on the discount factor: in particular, we might try to
serve the user at increasing levels of $x_t$ as long as we see they did not
abandon.  Surprisingly, our main result shows this is not the case: that in
fact, the {\em static} policy of maximizing one-step reward is optimal for this
problem.  Essentially, because the user abandons if the threshold is ever
crossed, there is no value to trying to actively learn the threshold.

In Section~\ref{sec:learning}, we consider how to adapt our results when $F$ and/or the reward function are
unknown.  In this case, the platform can learn over multiple user arrivals.  We
relate the problem to one of learning an unknown demand curve, and suggest an
approach to efficiently learning the threshold distribution $F$ and the reward function.

Finally in Section \ref{sec:feedback}, we consider a more general model with
``soft'' abandonment: after a negative experience, users may not abandon
entirely, but continue with the platform with some probability.  We
characterize the structure of an optimal policy to maximize expected discounted
reward on a per-user basis; in particular, we find that the policy adaptively
experiments until it has sufficient confidence, and then commits to a static
action.  We empirically investigate the structure of the optimal policy as
well.

\paragraph{Related work}

The abandonment setting is quite unique, and we are aware of only one
other work that addresses the same setting.
Independently from this work, \citet{LuMSOM} model
the abandonment problem using only two actions; the safe action and the
risky action. This naturally leads to rather different results.
There are some similarities with the mechanism design literature,
though there the focus is on strategic behavior by agents
\citep{rothschild1974two, myerson1981optimal, farias2010dynamic, pavan2014dynamic, lobel2017dynamic}.
As in this work, the revenue management literature considers
agents with heuristic behaviour, but the main focus is on dealing with a
finite inventory \citep{gallego1994optimal}.

It may seem that our problem is closely related to many problems in
reinforcement learning (RL) \citep{sutton1998reinforcement} due to the dynamic
structure of our problem.  However, there are important differences.  Our focus
is on personalization; viewed through the RL lens, this corresponds to having
only a single episode to learn, which is independent of other episodes (users).
On the other hand, in RL the focus is on learning an optimal policy using
multiple episodes where information carries over between episodes.  These
differences present novel challenges in the abandonment setting, and
necessitate use of the structure present in this setting.

Also related is work on safe reinforcement learning, where catastrophic states
need to be avoided \citep{Moldovan2012SafeEI, Berkenkamp2017SafeMR}.  In such a
setting, the learner usually has access to additional information, for example
a safe region is given.  Finally, we note that in our work, unlike in safe RL,
avoiding abandonment is not a hard constraint.



    \section{Threshold model}
\label{sec:model}

In this section, we formalize the problem of finding a personalized
policy for a single user without further feedback.

\subsection{Formal setup and notation}

We consider a setting where heterogeneous users interact with a platform at
discrete time steps indexed by $t$, and focus on the problem of finding a
personalized policy for a single user.  The user is characterized by
sequence of hidden thresholds $\{\theta_t\}_{t=0}^\infty$ jointly drawn from a
known distribution that models the heterogeneity across users.  At every time
$t$, the platform selects an action $x_t \in \mathbf{X} \subset \reals_+$ from
a given closed set $\mathbf{X}$.  Based on the chosen action $x_t$, the
platform obtains the random reward $R_t(x_t) \ge 0$.  The expected reward of
action $x$ is given by $r(x) = \E(R_t(x)) < \infty$, which we assume to be
stationary and known to the platform.\footnote{Section~\ref{sec:learning}
discusses the case when both $F$ and $r$ are unknown.} While not required for
our results, we expect $r$ to be increasing.  When the action exceeds the
threshold at time $t$, the process stops.  More formally, let $T$ be the
stopping time that denotes the first time the $x_t$ exceeds the threshold $\theta_t$:
\[
    T = \min \{ t : x_t > \theta_t \}.
\]

The goal is to find a sequence of actions $\{ x_t \}_{t=0}^\infty$ that
maximizes:
\[
    \E \left[ \sum_{t=0}^{T-1} \gamma^t R_t(x_t) \right],
\]
where $\gamma \in (0,1)$ denotes the discount factor.
We note that this expectation is well defined even if $T = \infty$,
since $\gamma < 1$.  We focus here on the discounted expected reward criterion.  An alternative approach is to consider maximizing average reward on a finite horizon; considering this problem remains an interesting direction for future work.

\subsection{Optimal policies}

Without imposing further restrictions on the structure of the stochastic
threshold process, the solution is intractable.  Thus, we first consider two
extreme cases: (1) the threshold is sampled at the start and then remains fixed across time; and (2) the thresholds are independent across time.  
Thereafter, we look at the robustness of the results when we
deviate from these extreme scenarios.

\paragraph{Fixed threshold}

We first consider a case where the threshold is sampled at the beginning of the horizon, but then remains fixed.  In other words, for all $t$, $\theta_t = \theta \sim F$.  Intuitively, we might expect that the platform might try to gradually learn this threshold, by starting with $x_t$ low and increasing it as long as the user does not abandon.  In fact, we find something quite different: our main result is that the optimal policy is a constant policy.
\begin{prop}
  \label{thm:fixed}
  Suppose the function
  and the function $x \to r(x) (1-F(x))$ has a unique optimum $x^* \in \mathbf{X}$.
  Then, the optimal policy is $x_t = x^*$ for all $t$.
\end{prop}
All proofs can be found in the supplemental material.

We sketch an argument why there exists a constant policy that is optimal.  Consider a policy
that is increasing and suppose it is optimal.\footnote{It is clear that the
optimal policy cannot be decreasing.} Then there exists a time $t$ such that
$x_t = y < x_{t+1} = z$.  Compare these two actions with the policy that would
use action $z$ at both time periods.  First suppose $\theta < y$; then the user
abandons under either alternative and so the outcome is identical.  Now
consider $\theta \geq y$; then by the optimality of the first policy, given
knowledge that $\theta \geq y$, it is optimal to play $z$.  But that means the
constant policy is at least as good as the optimal policy.

In the appendix, we provide another proof of the result using value iteration.  This proof also characterizes the optimal policy and optimal value exactly (as in the proposition).  Remarkably, the
optimal policy is independent of the discount factor $\gamma$.

\paragraph{Independent thresholds}

For completeness, we also note here the other extreme case: suppose the thresholds $\theta_t$ are drawn
independently from the same distribution $F$ at each $t$.  Then since there is no correlation
between time steps, it follows immediately that the optimal policy is a
constant policy, with a simple form.

\begin{prop}
  \label{thm:independent}
  Then the optimal policy under the independent threshold assumption is $x_t = x^*$
  for all $t$
  if
  \[
    x^* \in \arg\max_{x \in \mathbf{X}} \frac{r(x) (1-F(x))}{1-\gamma(1-F(x))}
  \]
  is the unique optimum.
\end{prop}

\paragraph{Robustness}

So far, we have considered two extreme threshold models and have shown
that constant policies, albeit different ones, are optimal.
In this section we look at the robustness of those results by
understanding what happens when we interpolate between the two sides by
considering an additive noise threshold model.
Here, the threshold at time $t$ consists of a fixed element and
independent noise:
$\theta_t = \theta + \epsilon_t$,
where $\theta \sim F$ is drawn once,
and the noise terms are drawn independently.
In general, the optimal policy in this model is increasing and intractable because
the posterior over $\theta$ now depends on all previous actions.
However, there exists constant policies that are close to optimal
in case the noise terms are either small or large, reflecting our preceding results in the extreme cases.

First consider the case where the noise terms are {\em small}.
In particular, suppose the error distribution has an arbitrary distribution over
a small interval $[-y, y]$.
\begin{prop}
  \label{thm:small_noise}
  Suppose $\epsilon_t \in [-y, y]$ and
  the reward function $r$ is $L$-Lipschitz.
  Then there exists a constant policy with
  value $V_c$ such that
  \[
      V^* - V_c \le \frac{2yL}{1-\gamma}
  \]
  where $V^*$ is the value of the optimal policy for the noise model,
  and $x^*$ is the optimal constant policy for the noiseless case.
\end{prop}
This result follows from comparing the most beneficial and detrimental scenarios;
$\epsilon_t = y$ and $\epsilon_t = -y$ for all $t$, respectively,
and nothing that in both cases the optimal policies are constant policies,
because thresholds are simply shifted.
We can then show that the optimal policy for the worst scenario achieves
the gap above compared to the optimal policy in the best case.
The details can be found in the appendix.

Similarly, when the noise level is sufficiently {\em large} with respect to the
threshold distribution $F$ there also exists a constant policy that is close to
optimal.  The intuition behind this is as follows.  First, if the noise level
is large, the platform receives only little information at each step, and thus
cannot efficiently update the posterior on $\theta$.  Furthermore, the high
variance in the thresholds also reduces the expected lifetime of any policy.
Combined, these two factors make learning ineffective.

We formalize this by comparing a constant policy to an oracle policy
that knows $\theta$ but not the noise terms $\epsilon_t$.
Let $G$ be the CDF of the noise distribution $\epsilon_t$ with
$\bar G$ denoting its complement: $\bar G(y) = 1 - G(y)$.
Then we note that for a given threshold $\theta$, the probability of survival
is $\bar G (x - \theta)$,
and thus the expected value for the constant policy $x_t = x$ for all $t$ is
\[
  \frac{\bar G(x - \theta) r(x)} {1-\gamma \bar G(x-\theta)}.
\]
Define the optimal constant policy given knowledge of the fixed part of the
threshold, $\theta$ by $x(\theta)$:
\[
  x(\theta) = \arg\max_x \frac{\bar G(x - \theta) r(x)} {1-\gamma \bar G(x-\theta)}.
\]
We can furthermore define the value of policy $x_t = x(\theta)$ when the
threshold is $\theta'$ by $v(\theta, \theta')$:
\[
  v(\theta, \theta') = \frac{\bar G(x(\theta) - \theta') r(x(\theta))} {1-\gamma \bar G(x(\theta)-\theta')}.
\]
We note that $v$ is non-decreasing in $\theta'$.
We assume that $v$ is $L_v$-Lipschitz:
\[
  |v(\theta, \theta') - v(s, \theta')| \le L_v |\theta - s|
\]
for all $\theta$ and $s$.
Note that noise distributions
$G$ that have high variance lead to a smaller Lipschitz constant.

To state our result in this case, we define an $\eta$-cover, which is a simple
notion of the spread of a distribution.
\begin{definition}
  An interval $(l, u)$ provides an $\eta$ cover for distribution $F$
  if $F(u) - F(l) > \eta$.
\end{definition}
In other words, with probability as least $1-\eta$, a random variable drawn from
distribution $F$ lies in the interval $(l, u)$.

\begin{prop}
  \label{thm:large_noise}
  Assume $r$ is bounded, and
  $\X$ is a continuous and connected space.
  Suppose $v$ defined above is $L_v$-Lipschitz,
  and there exists an $\eta$-cover for threshold distribution
  $F_\theta$ with width $w = u - l$.
  Then the constant policy $x_t = \frac{l + u}{2}$ with expected value
  $V_\theta$ satisfies
  \[
    V^* - V_\theta \le V_o - V_\theta \le \frac{L_v w}{2} + 2\frac{\eta B}{1-\gamma}.
  \]
\end{prop}

The shape of $v$, and in particular its Lipschitz constant $L_v$ depend
on the threshold distribution $F$ and reward function $r$.  As the noise distribution $G$ ``widens'', $L_v$ decreases.
As a result, the bound above is most relevant when the variance of $G$ is substantial relative to spread of $F$.

To summarize, our results show that in the extreme cases where the thresholds are drawn independently, or drawn once,
there exists a constant policy that is optimal.  Further, the class of constant policies is robust when
the joint distribution over the thresholds is close to either of these scenarios.

    \section{Learning thresholds}
\label{sec:learning}

Thus far, we have assumed that the heterogeneity across the population and the
mean reward function are known to the platform, and we have focused on
personalization for a single user.  It is natural to ask what the platform
should do when it lacks such knowledge, and in this section we show how the
platform can learn an optimal policy efficiently across the population.
We study this problem within the context of the fixed threshold model described above, as it naturally lends itself to development of algorithms that learn about population-level heterogeneity.
In particular, we give theoretical performance guarantees on a UCB type \citep{auer2002finite}
algorithm, and show that a variant based on MOSS \citep{MOSS} performs better in practice.
We also empirically show that an explore-exploit strategy performs well.

\paragraph{Learning setting}

We focus our attention on the fixed threshold model,
and consider a setting where $n$ users arrive sequentially,
each with a fixed threshold $\theta_u$ ($u = 1, \ldots, n$)
drawn from unknown distribution $F$ with support on $[0, 1]$.
To emphasize the role of learning from users over time, we consider a stylized setting where the platform interacts with one user at a time, deciding on all
the actions and observing the outcomes for this user, before
the next user arrives.
Inspired by our preceding analysis, we consider a proposed algorithm that uses a constant policy
for each user.
Furthermore, we assume that the rewards $R_t(x)$ are
bounded between $0$ and $1$, but otherwise drawn from an arbitrary distribution
that depends on $x$.

\paragraph{Regret with respect to oracle}

We measure the performance of learning algorithms against the oracle
that has full knowledge about the threshold distribution $F$ and the
reward function $r$, but no access to realizations of random variables.
As discussed in Section~\ref{sec:model}, the optimal policy for the oracle
is thus to play constant policy $x^* = \max_{x \in [0, 1]} r(x) (1 - F(x))$.
We define regret as
\begin{multline}
  \regret_n(A) = nr(x^*)(1-F(x^*)) \\
  - (1-\gamma) \sum_{u=1}^{n} \E \left[ \sum_{t=0}^{T_u - 1} \gamma^t r(x_{u, t}) \right]
\end{multline}
which we note is normalized on a per-user basis with respect to the discount factor $\gamma$.

\subsection{UCB strategy}

We propose a UCB algorithm \citep{auer2002finite} on a suitably discretized
space, and prove an upper bound on its regret in terms of the number of users.
This approach is based on earlier work by \citep{Kleinberg2003TheVO}[Section 3]
for learning demand curves.  Before presenting the details, we introduce the
UCB algorithm for the standard multi-armed bandit problem.

In the standard setting, there are $K$ arms, each with its own mean $\mu_i$.
At each time $t$, UCB($\alpha$) selects the arm with largest index $B_{i, t}$
\[
    B_{i, t} = \bar X_{i, n_i(t)} + \sigma \sqrt{\frac{2\alpha \log t}{n_i(t)}}
\]
where $n_i(t)$ is the number of pulls of arm $i$ at time $t$.
We assume $B_{i, t} = \infty$ if $n_i(t) = 0$.
The following lemma bounds the regret of the UCB index policy.
\begin{lemma}[Theorem 2.1 \citep{bubeck2012regret}]
    \label{thm:ucb}
    Suppose rewards for each arm $i$ are independent across multiple pulls,
    $\sigma$-sub-Gaussian and have mean $\mu_i$.
    Define $\Delta_i = \max_j \mu_j - \mu_i$.
    Then, UCB($\alpha$) attains regret bound
    \[
        \regret_n(UCB) \le
        \sum_{i : \Delta_i > 0}
            \frac{8 \alpha \sigma^2}{\Delta_i} \log n + \frac{\alpha}{\alpha - 2}.
    \]
\end{lemma}

\citet{Kleinberg2003TheVO} adapt the above result to the problem of demand curve learning.
We follow their approach:
Discretize the action space and then use the standard UCB approach to
find an approximately optimal action.
For each user, the algorithm selects a constant action $x_u$ and
either receives reward $R_u = 0$ if $x_u > \theta_u$ or
$R_u = \sum_{t=0}^{\infty} \gamma^t R_t(x_u)$.

We need to impose the following assumptions:
$\theta \in [0, 1]$, $0 \le R(x) \le M$ for some $M > 0$,
and the function $f(x) = r(x)D(x) = r(x) (1-F(x))$ is strongly convex
and thus has a unique maximum at $x^*$.

\begin{assumption}[Lemma 3.11 in Leighton and Kleinberg]
    \label{thm:concave}
    There exists constants $c_1$ and $c_2$ such that
    \[ c_1(x^* - x)^2 < f(x^*) - f(x) < c_2(x^* - x)^2\]
    for all $x \in [0, 1]$.
\end{assumption}

Using these assumptions, we can prove the main learning result.

\begin{theorem}
    \label{thm:ucb_upper}
    Suppose that $f$ satisfies the concavity condition above.
    Then UCB($\alpha$) on the discretized space with $K = O\left((n/\log n)^{1/4}\right)$ arms
    satisfies
    \[
        \regret_n(UCB) \le O\left(\sqrt{n \log n}\right)
    \]
    for all $\alpha > 2$.
\end{theorem}
The proof consists of two parts, first we use Lemma~\ref{thm:concave} to
bound the difference between the best action and the best arm in the
discretized action space.
Then we use Theorem~\ref{thm:ucb} to show that the learning strategy
has small regret compared to the best arm.
Combined, these prove the result.

It is important to note that the algorithm requires prior knowledge of the
number of users, $n$. In practice it is reasonable to assume that a platform
is able to estimate this accurately, but otherwise the well-known doubling trick
can be employed at a slight cost.

\subsection{Lower bound}

We now briefly discuss lower bounds on learning algorithms.
If we restrict ourselves to algorithms that play a constant policy for each
user, the lower bound in \cite{Kleinberg2003TheVO} applies immediately.

\begin{prop}[Theorem 3.9 in \cite{Kleinberg2003TheVO}]
  Any learning algorithm A that plays a constant policy for each user,
  has regret at least
  \[
    \regret_n(A) \ge \Omega(\sqrt{n})
  \]
  for some threshold distribution.
\end{prop}

Thus, the discretized UCB strategy is near-optimal in the class
of constant policies.

However, algorithms with dynamic policies for users can obtain more
information on the user's threshold and therefore more easily estimate
the empirical distribution function.
Whether the $O(\sqrt{n})$ lower bound carries over to dynamic policies
is an open problem.

\subsection{Simulations}
\label{sec:simulations}

In this section, we empirically compare the performance of the discretized UCB
against other policies.
For our simulations, we also include the MOSS
algorithm \citep{MOSS}, and an explore-exploit strategy.

\paragraph{MOSS}
\citet{MOSS} give a upper confidence bound algorithm that
has a tighter regret bound in the standard multi-armed bandit problem.
The MOSS algorithm is an index policy where the index for arm $i$ is given by
\[
  B_{i, t} =
    \bar X_{i, n_i(t)}
    + \sqrt{\left( \log \frac{t}{K n_i(t)} \right)_+ / n_i(t)}
\]
While the policy is quite similar to the UCB algorithm,
it does not suffer from an extra $\sqrt{\log n}$ term in the regret bound.
However, we cannot adapt the bound to the abandonment setting,
due to worse dependence on the number of arms.
In practice, we expect this algorithm to perform better than the UCB algorithm,
as it is a superior multi-armed bandit algorithm.

\paragraph{Explore-exploit strategy}
Next, we consider an explore-exploit strategy that first estimates an empirical distribution function,
and then uses that to optimize a constant policy.
For this algorithm, we assume that for zero reward, the learner
can observe $\theta_u$ for a particular user, which mimics a strategy
where the learner increases its action by $\epsilon$ at each time period
to learn the threshold $\theta_u$ of a particular user with arbitrary precision.
Because it directly estimates the empirical distribution function and does not
require discretization, it is better able to capture the structure of our model.

The explore-exploit strategy consists of two stages.
\begin{itemize}
  \item First, obtain $m$ samples of $\theta_u$ to
    find an empirical estimate of $F$, which we denote by $\hat F_m$
  \item For the remaining users, play constant policy
    $x_u = \arg\max r(x) (1 - \hat F_m(x))$
\end{itemize}
Note that compared to the previous algorithm,
we assume this learner has access to the reward function,
and only the threshold distribution $F$ is unknown.
If the signal-to-noise ratio in the stochastic rewards is large,
this is not unrealistic: the platform, while exploring,
is able to observe a large number of rewards and should therefore
be able to estimate the reward function reasonably well.

\paragraph{Setup}
For simplicity, our simulations focus on
a stylized setting; we observed similar results under
different scenarios.\footnote{Code to replicate the simulations under
  a variety of scenarios is available at \url{https://github.com/schmit/learning-abandonment}.}
We assume that the rewards are deterministic and
follow the identity function $r(x) = x$, and
the threshold distribution (unknown to the learning algorithm)
is uniform on $[0, 1]$.
For each algorithm, we run 50 repetitions for $n = 2000$ time steps,
and plot all cumulative regret paths.
For the discretized policies, we set $K \approx 2.5 \left(\frac{n}{\log n}\right)^{1/4} = 12$.
The explore-exploit strategy first observes $20 + 2\sqrt{n} = 110$ samples
to estimate $F$, before committing to a fixed strategy.

\paragraph{Results}

The cumulative regret paths are shown in Figure~\ref{fig:regret}.
We observe that MOSS, while having higher variance,
indeed performs better than the standard UCB algorithm,
despite the lack of a theoretical bound.

However, the explore-exploit strategy obtains the lowest regret.
First, since it is aware of the reward function, it has less uncertainty.
More importantly, the algorithm leverages the structure of the problem because
it does not discretize the action space and then treat actions independently.
Finally, we note that when rewards are stochastic, the UCB and MOSS are even worse
compared to explore-exploit, as they have to estimate the mean reward function,
while the explore-exploit strategy assumes it is given.

\begin{figure}
  \centering
  \includegraphics[width=\textwidth]{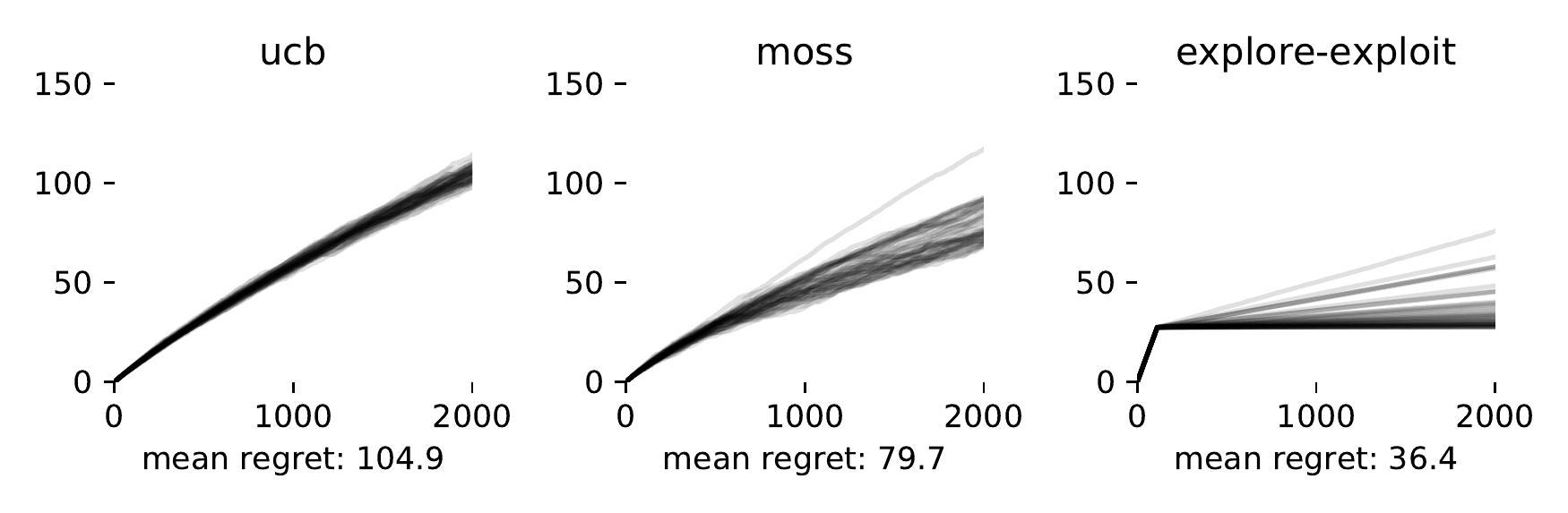}
  \caption{
    Cumulative regret plots for $r(x) = x$ and $F = U[0, 1]$.
  }
  \label{fig:regret}
\end{figure}

    \section{Feedback}
\label{sec:feedback}

In this section, we consider a ``softer'' version of abandonment, where the platform receives some feedback before the user abandons.
As example, consider optimizing the number of push notifications.
When a user receives a notification, she may decide to open the app,
or decide to turn off notifications.
However, her most likely action is to ignore the notification.
The platform can interpret this as a signal of dissatisfaction, and work to improve the policy.

In this section, we augment our model to capture such effects.
While the solution to this updated model is intractable,
we discuss interesting structure that the optimal policy exhibits:
{\em partial learning}, and the {\em aggressiveness} of the optimal policy.

\paragraph{Feedback model}
To incorporate user feedback, we expand the model as follows.  Suppose that whenever the current action $x_t$
exceeds the threshold (i.e., $x_t > \theta_t$),
then with probability $p$ we receive no reward
but the user remains, and with probability $1-p$ the user abandons.  Further, we assume that the platform at time $t$ both observes the reward $R(x_t)$,
if rewarded,
and an indicator $Z_t = \ind_{x_t > \theta_t}$.
This is equivalent to assuming that a user has geometrically distributed
\emph{patience}; the number of times she allows the platform to cross her threshold.

As before the goal is to maximize expected discounted reward.
Note that because the platform does not receive a reward when the
threshold is crossed, the problem is nontrivial even when $p=1$.
We restrict our attention to the single threshold model, where
$\theta$ is drawn once and then fixed for all time periods.

Figure~\ref{fig:tree} shows the numerically computed optimal policy
when the threshold distribution is uniform on $[0,1]$,
the reward function is $r(x) = x$, the probability of abandonment $p=0.5$
and $\gamma = 0.9$.
Depending on whether or not a feedback signal is received, the
optimal policy follows the green or the red line as we step through
time from left to right.

We note that one can think of the optimal policy as a form of bisection, though
it does not explore the entire domain of $F$. In particular it is conservative
regarding users with large $\theta$.  For example, consider a user with
threshold $0.9$.  While the policy is initially increasing and thus partially
personalizes to her threshold, $x_t$ does not converge to $0.9$, and in fact
never comes close.  We call this \emph{partial learning}; in the next section, we demonstrate that this is a key feature of the optimal policy in general.

\begin{figure}
  \centering
  \includegraphics[width=0.8\textwidth]{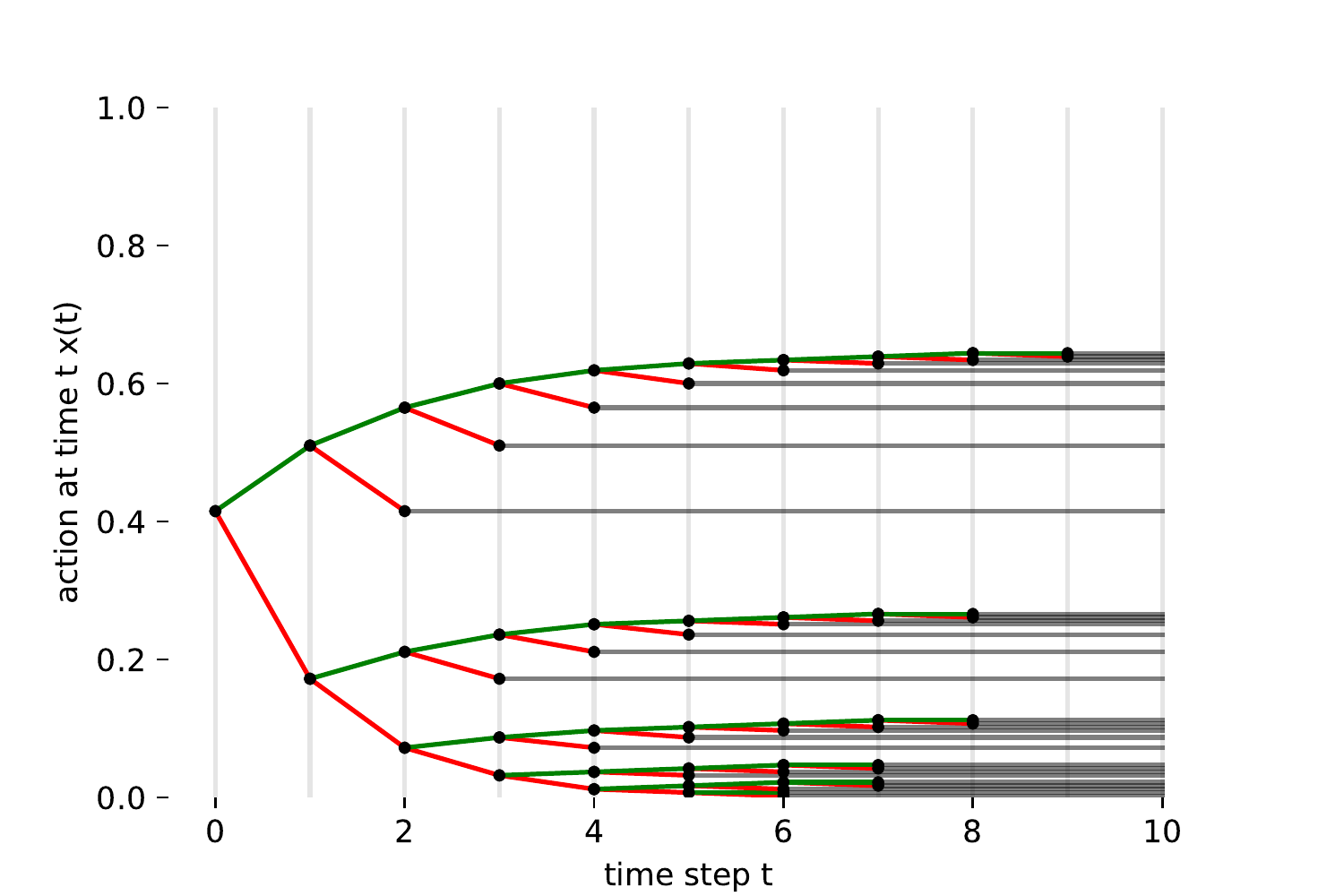}
  \caption{
Visualization of optimal policy when discount factor $\gamma = 0.9$
in the $p=0.5$ model.
Follow the tree from left to right, where if $Z_t = 0$ (reward obtained)
the next action follows from following the green line, and
if $Z_t = 1$, the optimal action is given by the point following the red line
if the user has not abandoned.
  }
  \label{fig:tree}
\end{figure}

\paragraph{Partial learning}

Partial learning refers to the fact that the optimal policy does not fully
reduce its uncertainty (the posterior) on $\theta$.
Initially, the policy learns about the threshold using a bisection-type search.
However, at some point (dependent on the user's threshold), further learning
is too risky and the optimal policy switches to a constant policy.
We note that this happens even when there is no risk of abandonment at all ($p=1$),
because at some point even the risk of losing a reward is not offset by potential
gains in getting a more accurate posterior on $\theta$.
Partial learning occurs under some regularity conditions on the
threshold distribution that ensures the posterior does not collapse,
and is Lipschitz as defined in the following paragraph.

Write $F_l^u$ for the posterior distribution over $\theta$ given
lower bound $l$ and upper bound $u$ based on previous actions
\[
  F_l^u(y) = \P(l + y < \theta \mid l < \theta < u) = \frac{F(u) - F(l+y)}{F(u) - F(l)}.
\]
We say the that the posterior distribution is non-degenerate if the
following condition holds:
\begin{definition}[Non-degenerate posterior distribution ]
  For all $\lambda > 0$, there exists a $\nu$ such that for all $l, u$ where $u - l <
  \nu$, $F_l^u(\epsilon) < 1 - \lambda \epsilon $ for $0 < \epsilon < \nu$.
\end{definition}
Thus, for sufficiently small intervals, the conditional probability decreases
rapidly as we move away from the lower bound of the interval.
Suppose $F$ is such that the posterior is non-degenerate and
is Lipschitz in the following sense.

\begin{assumption}[Lipschitz continuity of conditional distribution]
  There exists an $L' > 0$ such that
  for all intervals $[l, u]$ and all $0 < y < u - l$, we have
  \[
    p(y \mid l + \epsilon, u) - p(y \mid l, u) \le \epsilon L'.
  \]
\end{assumption}

We can use this assumption to show that the value function corresponding to the dynamic program
that models the feedback model is Lipschitz.

\begin{lemma}[Lipschitz continuity of value function]
  \label{thm:lipschitz_value}
  Consider a bounded action space $\mathbf{X}$.
  If $p$ is Lipschitz with Lipschitz constant $L_p$,
  and the reward function $r$ is bounded by $B$,
  there exists constant $L_V$ such that for all $l < u$
  \[
    V(l + \epsilon, u) - V(l, u) \le \epsilon L_V.
  \]
\end{lemma}

Using these assumptions, we can then prove that
the optimal policy exhibits partial learning,
as stated in the following proposition.

\begin{prop}
  \label{thm:partial_learning}
  Suppose $r$ is increasing, $L_r$-Lipschitz,
  non-zero on the interior of $\mathbf{X}$ and bounded by $B$.
  Furthermore, assume $p$ is non-degenerate and Lipschitz as defined above.
  For all $u \in \mathop{Int}(\mathbf{X})$ there exists an $\epsilon(u) >0$ such that for
  all $l$ where $u - l < \epsilon(u)$, the optimal action in state
  $(l, u)$ is $l$, that is
  \[
    V(l, u) = \frac{r(l)}{1-\gamma}.
  \]
  Furthermore, $\epsilon(u)$ is non-decreasing in $u$.
\end{prop}
We prove this result by analyzing the value function of the corresponding
dynamic program.
The result shows that at some point, the potential gains from a better posterior for the threshold
are not worth the risk of abandonment.
This is especially true when $\theta$ is quite likely under the posterior.
If, to the contrary, we belief the threshold is small, there is little to lose
in experimentation.
Note however that the result also holds for $p=1$, where there are only signals
and no abandonment.
In this case the risk of a signal (and no reward for the current timestep), outweights
(all) possible future gains.
Naturally, if the probability of override is small (i.e. $p$ is small),
the condition on $\lambda$ also weakens, leading to larger intervals of constant policies.


\paragraph{Aggressive and conservative policies}

Another salient feature of the structure of optimal policies in the feedback model
is the aggressiveness of the policy.
In particular, we say a policy is \emph{aggressive} if the first action $x_0$ is larger
than the optimal constant policy $x^*$ in the
absence of feedback (corresponding to $p=0$),
and \emph{conservative} if it is smaller.
As noted before, when there is no feedback, there is no benefit to adapting
to user thresholds.
However, there is value in personalization when users give feedback.

Empirically, we find that when there is low risk of abandonment,
i.e., $p \approx 1$, then the optimal policy is aggressive.
In this case, the optimal policy can aggressively target
high-value users because other users are unlikely to abandon immediately.
Thus the policy can personalize to high-value users in later periods.

However, when the risk of abandonment is large ($p \approx 0$)
and the discount factor is sufficiently close to one,
the optimal policy is more conservative than the optimal constant
policy when $p=0$.
In this case, the high risk of abandonment forces the policy
to be careful: over a longer horizon the algorithm can extract
value even from a low value user, but it has
to be careful not to lose her in the first few periods.
This long term value of a user with low threshold makes up
for the loss in immediate reward gained from aggressively
targeting users with a high threshold.
Figure~\ref{fig:first_action} illustrates this effect.
Here, we use deterministic rewards $r(x) = x$ and the
threshold distribution is uniform $F = U[0,1]$,
but a similar effect is observed for other distributions and reward functions
as well.

\begin{figure}
  \centering
  \includegraphics[width=0.8\textwidth]{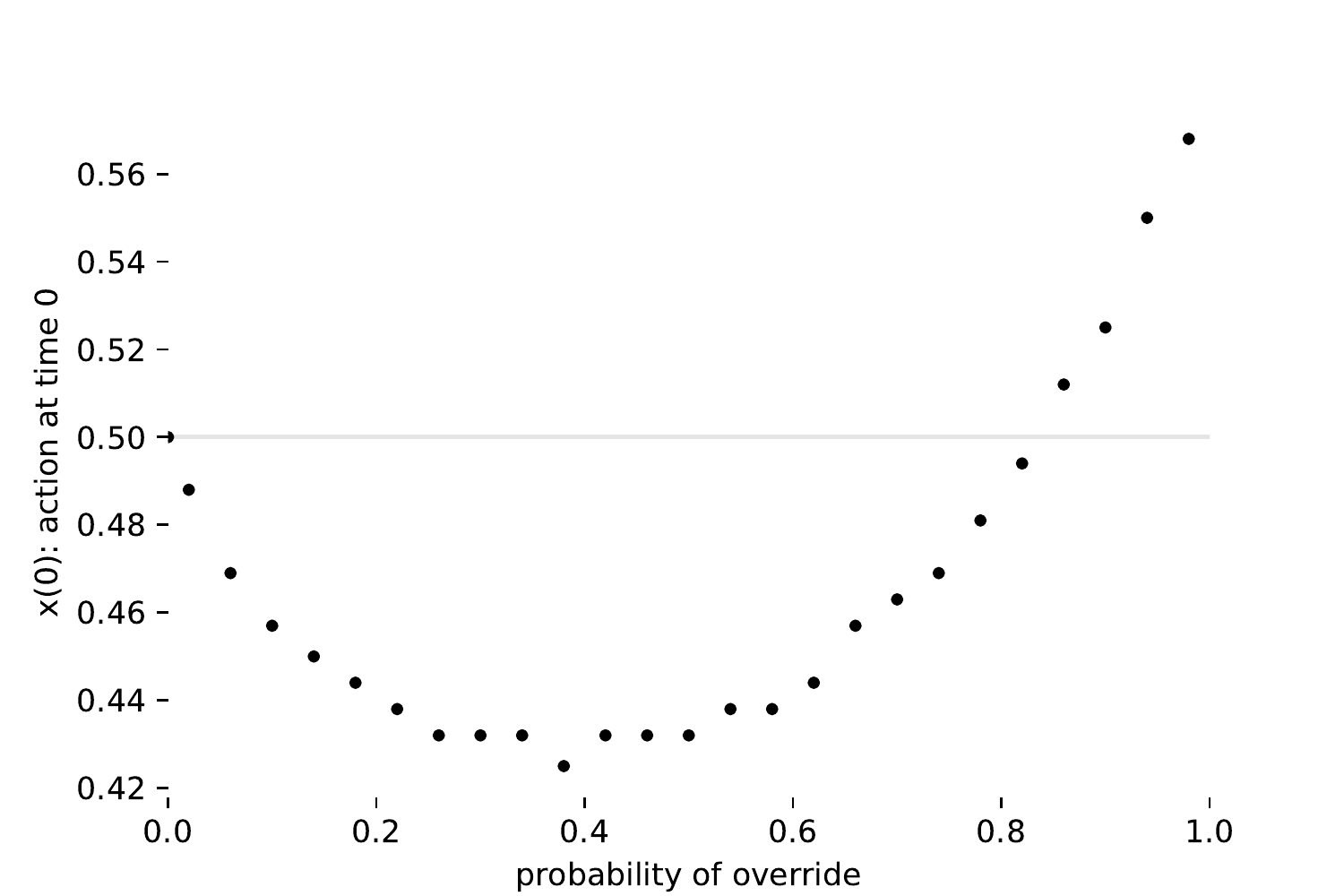}
  \caption{
    The relation between the override probability $p$ and the (approximate) optimal initial action $x_0$
when the discount factor $\gamma = 0.9$.
The artifacts in the plot are due to the discretization error from numerical computations.
  }
  \label{fig:first_action}
\end{figure}

    \section{Conclusion}

When machine learning algorithms are deployed in settings where
they interact with people, it is important to understand
how user behavior affects these algorithm.
In this work, we propose a novel model for personalization that takes into account
the risk that a dissatisfied user abandons the platform.

This leads to some unexpected results.  We show that constant policies are
optimal under fixed threshold and independent threshold models.  
We have shown that under small perturbations of these models,
constant policies are ``robust'' (i.e., perform well in the perturbed model),
though in general finding an optimal policy becomes intractable.

In a setting where a platform faces many users, but
does not know the reward function nor population distribution over
threshold, under suitable assumptions we have shown that
UCB-type algorithms perform well, both
theoretically by providing regret bounds and running simulations.
We also consider an explore-exploit strategy that is more efficient
in practice, but it requires knowledge of the reward function.

Feedback from users leads to more sophisticated optimal learning strategies
that exhibit partial learning; the optimal learning algorithm
personalizes to a certain degree to each user.
Also, we have found that the optimal policy is more conservative
when the probability of abandonment is high, and aggressive when
that probability is low.

\subsection{Further directions}

There are several interesting directions of further research
that are outside the scope of this work.

\paragraph{Abandonment models}
First, more sophisticated behaviour on user abandonment should be considered.
This could take many forms, such as a total \emph{patience budget}
that gets depleted as the threshold is crossed.
Another model is that of a user playing a learning strategy herself,
comparing this platform to one or multiple outside options.
In this scenario, the user and platform are simultaneously learning
about each other.

\paragraph{User information}
Second, we have not considered additional user information in
terms of covariates.
In the notification example, user activity seems like an important
signal of her preferences.
Models that are able to incorporate such information and are able
to infer the parameters from data are beyond the scope of this work
but an important direction of further research.

\paragraph{Empirical analysis}
This work focuses on theoretical understanding of the abandonment model,
and thus ignores important aspects of a real world system.
We believe there is a lot of potential to gain additional insight
from an empirical perspective using real-world systems with abandonment risk.

    \section{Acknowledgements}

The authors would like to thank Andrzej Skrzypacz,
Emma Brunskill, Ben van Roy, Andreas Krause, and Carlos Riquelme
for their suggestions and feedback.
This work was supported by the Stanford TomKat Center, and by the National Science Foundation under
Grant No. CNS-1544548.  Any opinions, findings, and conclusions or recommendations expressed in this
material are those of the author(s) and do not necessarily reflect the views of the National Science
Foundation.

    \bibliographystyle{plainnat}
    \bibliography{bib/bibliography}

    \appendix

\section{Proofs}
\label{sec:proofs}

\subsection{Threshold models}

\begin{proof}[Proof of Proposition~\ref{thm:fixed}]
  The proof follows from defining an appropriate dynamic program and
  solving it using value iteration.
  We will denote the state by $x$, denoting the best lower bound on $c$.
  In practice, if the process survives up to time $t$ ($T > t$)
  the state is $x = \max_{s \le t} x_s$.
  Furthermore, it is convenient to use the survival function $S(x) = 1 - F(x)$.

  It is easy to see that the optimal policy is non-decreasing,
  so we can restrict our focus to non-decreasing policies.

  The Bellman equation for the value function at state $x$ is given by
  \[
  V(x) = \max_{y \ge x} \frac{S(y)}{S(x)} (r(y) + \gamma V(y)).
  \]
  For convenience we define the following transformation $ J(x) = S(x)V(x) $
  and note that we can equivalently use $J$ to find the optimal policy.
  We now explicitly compute the limit of value iteration to find $J(x)$.
  Start with $J_0(x) = 0$ for all $x$ and note that the iteration takes the form
  \[
      J_{k+1} = \max_{y \ge x} S(y) r(y) + \gamma J_k(y) = \max_{y \ge x} p(y) + \gamma J_k(y).
  \]
  We prove the following two properties by induction for all $k > 0$:
  \begin{enumerate}
    \item $J_k(x) = p(x^*) \sum_{i=0}^{k-1} \gamma_i$ for all $x \le x^*$.
    \item $J_k(x) < J_k(x^*)$ for all $x > x^*$.
  \end{enumerate}
  The above is immediately true for $k = 1$.
  Now assume it is true for an arbitrary $k$, then
  \[
  J_{k+1}(x) = p(x^*) + \gamma J_k(x^*) \qquad \text{ for all } x \le x^*
  \]
  and
  \[
  J_{k+1}(x) < p(x^*) + \gamma J_k(x^*) = J_{k+1}(x^*) \qquad \text{ for all } x > x^*.
  \]
  The result follows from taking the limit as $k \to \infty$ and noting that for any state $x \le x^*$,
  it is optimal to jump to state $x^*$ (and stay there).
  We also immediately see that the value of the optimal policy thus is $p(x) / \gamma$, as required.
\end{proof}

\begin{proof}[Proof of Proposition~\ref{thm:independent}]
  It is immediate that the optimal policy must be constant; if the process
  survives $x_t = x$, then at time $t+1$ we face the same problem as at time
  $t$. So whatever action is optimal at time $t$, is also optimal at time $t+1$.
  Let $V(x)$ denote the value of playing $x_t = x$ for all $t$.
  Then the following relation holds
  \[
    V(x) = (1-F(x)) (r(x) + \gamma V(x))
  \]
  which leads to
  \[
    V(x) = \frac{r(x)(1 - F(x))}{1 - \gamma (1 - F(x))}.
  \]
  The result now follows immediately.
\end{proof}

\subsection{Robustness}

\begin{proof}[Proof of Proposition~\ref{thm:small_noise}]
  First we consider the constant policy $x_t = x^* - y$ for all $t$
  in the noiseless case.
  We note that
  \[
    r(x^* - y) S(x^* - y)
    \ge (r(x^*) - yL) S(x^*)
    \ge V(x^*) - yL
  \]
  where $V(x^*)$ is the value of the optimal constant policy for the noise-free model.

  Now let us consider the best possible noise model,
  then $\epsilon_t = y$ for all $t$.
  But this is equivalent to the noise-free model with the threshold shifted by $y$.
  Hence, we know that a constant policy is optimal.
  We can bound the value of this model by
  \begin{align}
    \max_x r(x) S(x-y)
    &= \max_x r(x + y) S(x) \\
    &\le \max_x ( r(x) + y L ) S(x) \\
    &= \max_x r(x) S(x) + y L S(x) \\
    &\le \max_x r(x) S(x) + y L\\
    &= V(x^*) + y L
  \end{align}
  Hence, this implies that the constant policy $x_t = x^* - y$
  is at most $\frac{2yL}{1-\gamma}$ worse than the optimal policy for the
  most optimistic noise model.
\end{proof}

\begin{proof}[Proof of Proposition~\ref{thm:large_noise}]
  Let $\bar \theta$ be the midpoint of the $\eta$ cover, $c = \frac{l+u}{2}$.
  Now we bound the expected value of an oracle policy, i.e.
  a policy that knows the true threshold $\theta^*$ as follows
  \begin{align*}
    \E(v(\theta^*, \theta^*))
    &\le \frac{2 \eta B}{1-\gamma} + \int_l^u v(\theta^*, \theta^*) dF_\theta\\
    &\le \frac{2 \eta B}{1-\gamma} + \int_l^u v(\theta^*, \theta^*) + L|\bar \theta - \theta^*|dF_\theta\\
    &\le \frac{2 \eta B}{1-\gamma} + \int_l^u v(\bar \theta, \theta^*) + L\frac{u - l}{2} dF_\theta\\
    &\le \E(v(\bar \theta, \theta^*)) + \frac{2 \eta B}{1-\gamma} + (1-\eta) \frac{Lw}{2}
  \end{align*}
  which completes the proof.
\end{proof}

\subsection{Learning}

\begin{proof}[Proof of Proposition~\ref{thm:ucb_upper}]
    Due to the discretization, the proof consists of two parts.
    First, we show that the policy that plays the best arm $i^*$
    suffers small regret with respect to the optimal policy.
    Then we use the UCB regret bound to show that
    the learning strategy has low regret with respect to the
    playing arm $i^*$.
    Thus we can decompose regret into
    \[
        \regret(UCB) = \regret_D + \regret_U
    \]
    where the first term corresponds to the discretization error and the second
    from the learning policy.
    Due to the time horizon and discounting, we write

    Let $x^*$ be the optimal strategy, i.e. it maxizimes $r(x)D(x)$.
    Then the discretization error from playing $i^*/K$, by Assumption~1
    is
    \[
        \regret_D \le \frac{c_2 n}{2K^2}  = \frac{c_2 \sqrt{n \log n} }{2}.
    \]
    Thus, the error due to the discretization is small.

    Now let us bound the UCB regret with respect to action $i^* / K$.
    As Kleinberg and Leighton (2003) note, the assumption that the pulls of different
    arms are independent is not used in the proof.
    Thus we can apply Lemma 5.
    First, we show that the arms are sub-Gaussian.
    Since the rewards are bounded by $1$ and independent across time,
    straightforward calculation shows that
    \[
      \var{(1-\gamma)\sum_{t=0}^\infty \gamma^t R_t(x)}
        = \frac{(1-\gamma)^2}{4 (1-\gamma^2)}
        \le \frac{1}{4}.
    \]
    Then using the law of total variance, conditioning on the event $x < \theta_u$,
    the variance of the total obtained reward for user $u$, $R_u$, can be bounded by
    \begin{align}
        \var{ R_u }
        &= \E(\var{ R_u \mid \theta_u }) +
        \var{ \E(R_u \mid \theta_u) }\\
        &= \frac{(1-F(x_u)) M^2}{4} +
          \left( r(x_u) \right)^2 F(x_u) (1-F(x_u))\\
        &\le M^2/2
    \end{align}
    Thus we find that the reward for users is sub-Gaussian
    with parameter $ \sigma = \frac{M^2}{2} $.

    Recall the UCB regret bound
    \[
        \regret(UCB) \le
        \sum_{i : \Delta_i > 0}
            \frac{8 \alpha \sigma^2}{\Delta_i} \log n + \frac{\alpha}{\alpha - 2}.
    \]
    We now focus on the $\sum_{i=1 : \Delta_i > 0}^K \frac{1}{\Delta_i}$ term.
    Let $\Delta_{(1)} \le \Delta_{(2)} \le \ldots \le \Delta_{(K-1)}$ denote
    the ordered gaps with respect to the optimal arm.
    Note that for $j \ge 2$, we know $\Delta_{(j)} > c_1(\frac{j}{2K})^2$
    due to Assumption~1.
    However, for the smallest gap, we only know $0 \le \Delta_{(1)} \le \frac{c_2}{K^2}$,
    depending how close $i^*/K$ is to $x^*$.
    We thus obtain
    \begin{align}
        \sum_{i=1}^K \frac{1}{\Delta_i}
        &= \sum_{i=1}^{K-1} \frac{1}{\Delta_{(i)}} \\
        &= \frac{1}{\Delta_{(1)}} + \sum_{i \ge 2} \frac{1}{\Delta_{(j)}}\\
        &\le \frac{1}{\Delta_{(1)}} + \frac{4K^2}{c_1} \sum j^{-1}\\
        &\le \frac{1}{\Delta_{(1)}} + \frac{2\pi^2}{3 c_1} K^2
    \end{align}
    Thus regret is bounded by
    \[
        \regret_U \le
        \frac{8 \alpha \sigma^2 \log n}{\Delta_{(1)}}
        + \frac{16 \alpha \sigma^2 \pi^2}{3 c_1} (K-2)^2 \log n
        + K \frac{\alpha}{\alpha - 2}
    \]
    However, the regret from due to playing the second best action
    is trivially bounded by $n \Delta_{(1)}$.
    Thus, we can bound the worst case when $\Delta_{(1)} = 4 \sqrt{ \log n / n }$.
    This leads to a bound of
    \[
        \regret_U \le
        2\alpha \sigma^2 \sqrt{n\log n}
        + \frac{16 \alpha \sigma^2 \pi^2}{3 c_1} (K-2)^2 \log n
        + K \frac{\alpha}{\alpha - 2}
    \]
    since there are $K = (n/\log n)^{1/4}$ arms, we get
    \[
        \regret_u \le
        2\alpha\sigma^2 \sqrt{n\log n}
        + \frac{16\alpha \sigma^2 \pi^2}{3 c_1} \sqrt{n \log n}
        + o(\sqrt{n \log n})
    \]
    Combining this with the bound on $\regret_D$ completes the proof.
\end{proof}

\subsection{Feedback}

\begin{proof}[Proof of Lemma~\ref{thm:lipschitz_value}]
  The Bellman equation of the dynamic program for the feedback model can be
  written as:
  \[
    V(l, u) =
    \max_{l \le y \le u}
    \frac{F(u) - F(y)}{F(u) - F(l)} (r(y) + \gamma V(y, u))
    + \frac{F(y) - F(l)}{F(u) - F(l)} \gamma V(l, y)
  \]
  where $l$ and $u$ are the lower bounds and upper bounds on $c$ based on the history.

  Note that $V$ is finite and therefore value iteration
  converges pointwise to $V$.
  We use induction on the value iterates to find the Lipschitz constant for $V$.
  Let $V_0, V_1,\ldots$ indicate the value iterates.
  Since $V_0(l, u) = 0$ for all states $(l, u)$, the Lipschitz constant for $V_0$,
  denoted by $L_0 = 0$.
  We further claim that $L_{n+1} = L_p \frac{B}{1-\gamma} + \beta \gamma L_n$.
  Suppose this is true for $n = 1,\ldots,i-1$, then for $n=i+1$ we consider
  state $(l + \epsilon, u)$ and write $x^*$ for the optimal action
  in that state, and $y^* = x^* - l$.
  Then
  \begin{multline}
    V_{i+1}(l, u) \ge p(y^* \mid l, u) (r(x^*) + \gamma V(x^*, u))
      + (1-p(y^* \mid l, u)) \beta \gamma V(l, x^*)
    \end{multline}
  Also, $V(l, x^*) \le V(l, u)$.
  Then we find
  \begin{multline}
    V_{i+1}(l + \epsilon, u) - V_{i+1}(l, u)
    \le \left[p(y^* \mid l + \epsilon, u) - p(y^* \mid l, u)\right]
      (r(x^*) + \gamma V_i(x^*, u) )\\
    + (1-p(y^* \mid l + \epsilon, u)) \beta \gamma V_i(l+\epsilon, x^*)
    - (1-p(y^* \mid l, u) \beta \gamma V_i(l, x^*)
  \end{multline}
  Using the Lipschitz continuity of $p$ we can bound
  \[
    p(y^* \mid l + \epsilon, u) - p(y^* \mid l, u) \le \epsilon L_p.
  \]
  Then note that
  \[
    r(x^*) + \gamma V(x^*, u) \le \frac{B}{1-\gamma}
  \]
  and for the final two terms we note
  \begin{multline}
    (1-p(y^* \mid l + \epsilon, u)) \beta \gamma V_i(l+\epsilon, x^*)
    - (1-p(y^* \mid l, u) \beta \gamma V_i(l, x^*)\\
    \le \beta \gamma (V_i(l+\epsilon, x^*) - V_i(l, x^*))
    \le \beta \gamma \epsilon L_i
  \end{multline}
  where we use the inductive assumption.
  Because $l, u$ and $\epsilon$ are arbitrary, we see that
  \[
    L_n \le \frac{L'B}{(1-\beta\gamma)(1-\gamma)}.
  \]
  which implies $V$ is Lipschitz.
\end{proof}

\begin{proof}[Proof of Proposition~\ref{thm:partial_learning}]
  First we note that by Lemma~8, $V$ is Lipschitz, and
  we write $L_v$ for its Lipschitz constant.
  Fix $u$, and consider a state $(u-\nu, u)$ for some $\nu > 0$.
  For notational convenience, for action $x$ we write $y = x - (u - \nu)$
  for the difference from the lower bound.
  We also use the shorthand $l = u-\nu$ and $p(y) = p(y \mid l,u)$.
  We can upperbound the value function by
  \begin{align}
    V(l, u) &=
    \max_y p(y) [r(x) + \gamma V(x, u)] + (1-p(y)) \beta\gamma V(l, x)\\
    &\le p(y)[r(l) + L_r y + \gamma V(l, u)\\
    &\quad+ \gamma L_v y] + (1-p(y)) \beta \gamma V(l, u)\\
    &\le (1-\lambda(\nu)y) [r(l) + \gamma V(l, u) + Ly]\\
    &\quad+ \lambda(\nu)y\beta\gamma V(l, u)
  \end{align}
  where we write $L = L_r + \gamma L_v$ and use the non-degeneracy of $p$.
  The derivative for the above expression with respect to $y$ is
  \begin{multline}
    (1-2\lambda(\nu))Ly + L - \lambda(\nu) r(l) - \gamma \lambda(\nu) (1-\beta) V(l, u)\\
    \le (1-2\lambda(\nu))Ly + L - \lambda(\nu) r(l).
  \end{multline}
  Since $r(l) > 0$ for all $l \in \mathop{Int}\mathbf{X}$, for $\nu$ sufficiently
  small this derivative is negative for all $y\ge 0$.
  To complete the proof, we need this upperbound to be tight at $y=0$,
  which follows immediately
  \begin{multline}
    \left. (1-\lambda(\nu)y) [r(l) + \gamma V(l, u) + Ly] + \lambda(\nu)y\beta\gamma V(l, u) \right\rvert_{y=0} = \\
      r(l) + \gamma V(l, u) \ge \frac{r(l)}{1-\gamma}.
  \end{multline}
  Since $r$ is increasing, it follows immediately that $\epsilon(u)$ is non-decreasing
  in $u$.
\end{proof}

\end{document}